\documentclass[lettersize,journal]{IEEEtran}
\usepackage{amsmath,amsfonts}
\usepackage{algorithmic}
\usepackage{algorithm}
\usepackage{array}
\usepackage[caption=false,font=normalsize,labelfont=sf,textfont=sf]{subfig}
\usepackage{textcomp}
\usepackage{stfloats}
\usepackage{url}
\usepackage{verbatim}
\usepackage{graphicx}
\usepackage{cite}

\usepackage[utf8]{inputenc}
\usepackage[english]{babel}
\usepackage[usenames]{color}
\usepackage{amsmath}
\usepackage{amssymb}
\usepackage{moreverb}
\usepackage{listings}

\def\t{{\boldsymbol \tau}}
\def\f{\mathbf{f}}

\def\th{{\boldsymbol \theta}}
\def\tt{{\boldsymbol {\tilde \theta}}}
\def\dtt{{\boldsymbol {\dot{\tilde \theta}}}}

\def\dth{{\boldsymbol {\dot \theta}}}

\def\tq{{\bf {\tilde q}}}	
\def\dtq{{\bf {\dot {\tilde q}}}}

\def\q{{\bf q}}
\def\dq{{\bf \dot q}}
\def\ddq{{\bf \ddot q}}

\def\x{{\bf x}}

\def\calH{{\mathcal H}}

\def\calK{{\mathcal K}}
\def\calU{{\mathcal U}}
\def\calF{{\mathcal F}}

\def\hal{{1 \over 2}}
\def\rea{\mathbb{R}}

\newtheorem{prop}{Proposition}

\begin{document}
	
	\title{Finite-Time Teleoperation of Euler-Lagrange\\ Systems via Energy-Shaping}
	
	\author{Lazaro F. Torres, Carlos I. Aldana, Emmanuel Nu\~no and Emmanuel Cruz-Zavala
		\thanks{This work has been supported by the Mexican SECIHTI Grant CBF2023-2024-1964. The work of L. Torres has been supported by a SECIHTI Scholarship,  Grant 1319340.\\ L. Torres, C. Aldana, E. Nu\~no and E. Cruz-Zavala are with the Department of Computer Science, CUCEI, University of Guadalajara, Mexico {(e-mail: lazaro.torres7987@alumnos.udg.mx, \{ivan.aldana, emmanuel.nuno, emmanuel.cruz1692\}@academicos.udg.mx})}
	}
	
\maketitle
	
	\begin{abstract}
This paper proposes a family of finite-time controllers for the bilateral teleoperation of fully actuated nonlinear Euler-Lagrange systems. Based on the energy-shaping framework and under the standard assumption of passive interactions with the human and the environment, the controllers ensure that the position error and velocities globally converge to zero in the absence of time delays. In this case, the closed-loop system admits a homogeneous approximation of negative degree, and thus the control objective is achieved in finite-time. The proposed controllers are simple, continuous-time proportional-plus-damping-injection schemes, validated through both simulation and experimental results.
\end{abstract}
	
	
\section{Introduction}

{
Bilateral teleoperation systems consist of a local and a remote robot manipulator that exchange information through a communication channel, enabling a human operator to control the remote robot in an environment subjected to unknown external forces \cite{Basanez2023}. In bilateral teleoperation, the roles of local and remote manipulators are interchangeable, and the primary control objective is to ensure position synchronization between both robots \cite{Hokayem:06}.

Recent technological advances have expanded the use of teleoperation systems across various applications, particularly in the medical industry \cite{Dong2018, Atashzar2017}. This has driven the development of advanced control strategies, including sliding-mode control \cite{Tang2019, liu_2019}, neural networks \cite{Yang20171, Chen2020}, adaptive control \cite{Kebria2020, SARRAS20144817}, and passivity-based control \cite{Atashzar2017, Zhai2014}. Among these, proportional-plus-damping (P+d) schemes stand out for their simplicity. They were initially introduced for bilateral teleoperation with time delays \cite{Nuno2008} and later extended to scenarios without velocity measurements and with actuator saturation \cite{Zhai2014}. However, these schemes only guarantee asymptotic stability, meaning that the control objective is achieved only as time approaches infinity.

This work focuses on ensuring finite-time (FT) convergence using nonlinear PD-like schemes. Compared to traditional (linear) PD-like controllers, FT schemes offer better robustness under uncertain conditions, faster transient responses, and higher precision \cite{Orlov2009, Bhat2005}, all of which are desirable characteristics in teleoperation systems. An FT controller for teleoperation systems must ensure bounded position errors and velocities in the presence of human or environment forces, and then drive the manipulators to a consensus position in finite-time once these forces vanish.

Current control schemes that ensure FT stability in teleoperation systems often rely on sliding-mode techniques \cite{Dao2021, Nguyen2021}, sometimes combined with neural networks \cite{Wang20221, Zhang2018, Zhang2022} or disturbance observers \cite{Dong2018, Zhang2023}. However, these methods suffer from the chattering effect, which can damage actuators and limit their applicability \cite{7600456}. Continuous-time FT controllers remain rare, with some solutions proposing variable-gain strategies \cite{Gu2024, Zhang2021}. Nonetheless, these works only guarantee convergence to a neighborhood of the origin.

One way of proving FT stability of the origin is to show that the origin is asymptotically stable (AS), and then prove that the closed-loop system admits a homogeneous approximation (HA) of negative degree \cite{Bacciotti2005}. Along this line, \cite{Yang2022} proposes a P+d control scheme that is claimed to ensure FT stability of the origin even in the presence of delayed communications. However, the conclusions in the work \cite{Yang2022} are not well-founded due to the following reasons: 1) in the presence of time delays, the resulting closed-loop system is not autonomous, and therefore the standard homogeneity tools employed in such work cannot be applied; and 2) the proof relies on Barbalat's Lemma, which requires differentiable accelerations, implying that the controller itself must be differentiable. Nevertheless, in order to achieve FT stability the controller cannot be differentiable, as is the case in \cite{Yang2022}. Thus, the result in \cite{Yang2022} does not hold. In addition, we emphasize that the control scheme reported in \cite{Yang2022} is not saturated and relies on velocity measurements, which limits its practical applicability. Therefore, the control of bilateral teleoperation systems via output-feedback schemes that avoid actuator saturation is still an open problem.
 
In this paper, we propose a family of continuous FT controllers for bilateral teleoperation systems based on the energy-shaping framework \cite{Ortega1998}. Our control design consists of shaping the potential energy of the robots, modeled as Euler-Lagrange (EL) systems, such that the origin is the unique equilibrium point, and then injecting damping to ensure that the origin is asymptotically stable when there is no time delay and the external forces are zero. If velocities are not available, damping is injected in the controller dynamics \cite{Ortega1998,Santibanez1998}. Assuming passive interactions with the human and the environment forces, we prove that the position errors and the velocities remain bounded. Then, when these forces vanish, we establish global FT stability of the origin using homogeneity tools. The resulting controllers are simple, continuous P+d schemes that avoid chattering. Our contributions include:
\begin{enumerate}
\item A novel family of FT controllers with continuous control laws for bilateral teleoperation systems that result in a continuous closed-loop.
\item A controller that does not require velocity measurements.
\item Controllers that avoid actuator saturation.
\item Experimental validation using 6-degrees-of-freedom robots.
\end{enumerate}
It is worth mentioning that the proposed design does not explicitly account for delays, although, as shown in the experiments, the controllers can effectively handle them in practice.

The paper is organized as follows: Section II covers the notation and the preliminaries; Section III introduces the dynamic model and the structure of the energy-shaping controllers; Section IV presents four examples of the controller design; Section V provides complementary remarks; Section VI presents numerical and experimental results; and finally, Section VII outlines the conclusions of our work.
} 
	
\section{Preliminaries}

This paper uses the following notation. $\mathbb{R}:=(-\infty, \infty)$, $\mathbb{R}_{>0}:=(0, \infty)$, $\mathbb{R}_{\geq 0}:=[0, \infty)$, $\mathbb{N}:= \{1,2,3,\dots\}$ and $\bar n:=\{1, 2, . . . , n\}$ for $n\in\mathbb{N}$. $\mathbf{I}_n \in \rea^{n\times n}$ denotes the $n\times n$ identity matrix, and $\mathbf{1}_{n} \in \rea ^{n}$ defines de vector of $n$ elements equal to one. For any $m\in \mathbb{N}$ and any $\delta\in\rea_{>0}$, $B_\delta:=\{{\bf x}\in \rea^m:\|{\bf x}\|<\delta\}$ and $S^{m-1}_\delta:=\{{\bf x}\in\rea^m:\|{\bf x}\|=\delta\}$ are an open ball and an $m-1$ sphere, centered at the origin with radius $\delta>0$, respectively. { In particular, $S^{m-1}:=S^{m-1}_{1}$.} A function $\mathbf{f}(t): \rea \mapsto \rea ^m$ is said to be of class $\mathcal{C}^{k}$, for $k \in \mathbb{N}$, if its derivatives $\mathbf{\dot f}$, $\mathbf{\ddot f}$, ..., $\mathbf{f}^{(k)}$ exist and are continuous. For all $x\in\rea$, $|x|$ is its absolute value. For $\x\in\rea^m$, $m\in \mathbb{N}$, $\|\x\|$ stands for its Euclidean norm. For any $\mathbf{x} \in \rea ^m$, $\nabla_{\mathbf{x}}:=[\partial_{x_1},\cdots, \partial_{x_m}]^{\top}$ stands for the gradient operator of a scalar function, where $\partial_{x_j}:=\frac{\partial}{\partial_{x_j}}$ and $j\in \bar{m}$. Operator ${\rm col}\{\cdot\}$ denotes a column vector concatenation. {Operator $\otimes$ denotes the Kronecker product.} {The argument of all time-dependent signals is omitted, e.g., $\q \equiv \q(t)$. The subscript $i\in\{l, r\}$ is used to denote the local ($l$) and remote $(r)$ manipulators.} For any $x\in\rea$ and $p>0$, we define the signed power function $\lceil x\rfloor^p:\rea\mapsto\rea$ as a strictly increasing odd (continuous) function given by $\lceil x\rfloor^p:=|x|^p{\rm sign}\left(x\right)$, where ${\rm sign}\left(x\right)$ is the standard {sign function}, and meets the properties described in \cite{Cruz20201}.	A ($p$, $\delta$)-saturation function ${\rm sat}_{\delta}(\lceil x \rfloor^p):\rea \mapsto \rea$, $p$, $\delta \in \rea_{>0}$, is a strictly increasing odd function defined by:
	\begin{equation}
		{\rm sat}_{\delta}(\lceil x \rfloor^p) =  \left\{\begin{matrix} \lceil x \rfloor^p& \text{if} & | x | < \delta, \\
		\delta^{p}{\rm sign}(x) & \text{if} & | x | \geq \delta,\end{matrix}\right. 
	\end{equation}
	this saturation function satisfies the following property:
	
\textit{Property (\textbf{P1}):} For all $x\in \rea$ and $p$, $\delta\in\rea _{>0}$, $\lceil{\rm sat}_{\delta}( x )\rfloor^p={\rm sat}_{\delta}(\lceil x \rfloor^p)$ and $s(x,\delta,p):=\int_{0}^x {\rm sat}_{\delta}(\lceil x \rfloor^p)$, $s(x)\in\mathcal{C}^1$, where:
		\begin{equation}
		\label{sfunc}
			s(x,\delta,p) = \left\{\begin{matrix} \frac{1}{p+1}|x|^{p+1}& \text{if} & | x | < \delta, \\
			\delta^{p}|x| - \frac{p}{p+1}\delta ^{p+1} & \text{if} & | x | \geq \delta.
			\end{matrix}\right.  
		\end{equation}
	$\hfill \triangleleft$
	
	Note that $\delta^{p}|x| - \frac{p}{p+1}\delta ^{p+1} \geq \frac{1}{p+1}\delta ^{p}|x|$ for all $|x|\geq \delta$.
	When the signed power and ($p$, $\delta$)-saturation functions are vector-valued, we consider them to be applied element-wise, i.e., for all $\mathbf{x}\in\rea^m$, we have $\lceil\mathbf{x}\rfloor^p:=[\lceil x_1 \rfloor^p, ..., \lceil x_m \rfloor^p]^{\top}$ and ${\rm sat}_{\delta}(\lceil \mathbf{x} \rfloor^p):=[{\rm sat}_{\delta}(\lceil x_1 \rfloor^p), ..., {\rm sat}_{\delta}(\lceil x_m \rfloor^p)]^{\top}$, respectively.
	
	\subsection{Finite-Time Stability and Homogeneous Systems}

	Consider the following dynamical system described by
	\begin{equation}
		\label{eq:ftconspts}
		{\bf \dot x}={\bf f}({\bf x}), \quad {\bf x}(0)={\bf x}_{0},
	\end{equation}
	where ${\bf x}\in \rea^m$ is the state vector. {Assume that: ${\bf f}: \rea^m\mapsto \rea^m$ is a continuous vector field; ${\bf x}={\bf 0}$ is an equilibrium point of (\ref{eq:ftconspts}), i.e., ${\bf f(0)=0}$; and solutions to \eqref{eq:ftconspts}, denote as $\x(t,\x_0)$, exist and are unique in forward time.}

	{\it Definition 1}\cite{Bhat2000}, \cite{Bhat2005}: {The point ${\bf x}={\bf 0}$} is Finite-Time Stable (FTS) if it is Lyapunov stable and there exists a locally bounded function $T:B_{\delta}\mapsto \mathbb{R}_{\geq0}$ such that for each ${\bf x}_0\in B_\delta\setminus \{\bf 0\}$, any solution ${\bf x}(t,{\bf x}_0)$ of (\ref{eq:ftconspts}) {	is defined on $t\in[0,T({\bf x}_{0}))$}, {$\lim_{t\to T(\x_{0})} {\bf x}(t,{\bf x}_0)={\bf 0}$}  and ${\bf x}(t,{\bf x}_0)={\bf 0}$ for all $ t\geq T({\bf x}_0)$. If $B_\delta=\rea^m$, $\mathbf{x=0}$ is globally FTS. $\hfill \triangleleft$ 
	
	{We recall the concept of weighted homogeneity that allows studying the finite-time stability properties of  system (\ref{eq:ftconspts}).}
	
	{\it Definition 2  \cite{Bacciotti2005}}: Let $r_j>0$, $j\in \bar m$, be the weights of the elements $x_j$ of ${\bf x}\in \rea^m$ and define the vector of weights as $\textbf{r}:=[r_1,...,r_m]^{\top}\in \rea^m$. Let $\Delta_{\epsilon}^{{\bf r}}$ be the dilation operator such that $\Delta_{\epsilon}^{{\bf r}}{\bf x}:=[\epsilon^{r_{1}}x_{1},...,\epsilon^{r_m}x_m]^\top$. A function $V:\rea^m\mapsto \rea$ (respectively a vector field ${\bf f}:\rea^m\mapsto \rea^m$) is said to be ${\bf r}$-homogeneous of degree $l_r\in \rea$, or $({\bf r},l_r)$-homogeneous for short, if for all $ \epsilon\in \rea_{>0}$ and for all ${\bf x}\in \rea^m$, the equality $V(\Delta_{\epsilon}^{{\bf r}}{\bf x})=\epsilon^{l_r}V({\bf x})$ (respectively, ${\bf f}(\Delta_{\epsilon}^{{\bf r}}{\bf x})=\epsilon^{l_r}\Delta_{\epsilon}^{{\bf r}}\mathbf{f(x)}$) holds. System (\ref{eq:ftconspts}) is called $({\bf r},l_r)$-homogeneous if the vector field ${\bf f}$ is $({\bf r},l_r)$-homogeneous. $\hfill \triangleleft$
	
	The following property is fundamental for our solution:
	
	\textit{Property (\textbf{P2}):} Let $V$ and $\mathbf{f}$ be $\mathbf{r}$-homogeneous of degree $l_V \in \rea_{>0}$ and $l_f\in\rea$, respectively. The Lie derivative of $V(\mathbf{x})$ along the vector  field $\mathbf{f}(\mathbf{x})$, $L_\mathbf{f}V := \nabla_{\mathbf{x}}V(\mathbf{x})^{\top}\mathbf{f}(\mathbf{x})$, is $(\mathbf{r}, l_V+l_f)$-homogeneous. In fact, $\partial_{x_j}V(\mathbf{x})$ is $(\mathbf{r},l_V-r_j)$-homogeneous. $\hfill \triangleleft$
	
{ Local stability properties of $\x={\bf 0}$, of (\ref{eq:ftconspts}), can be studied using homogeneous approximations \cite{Hermes1991, Bacciotti2005}. Here, we rephrase Corollary 5.5 from \cite{Bacciotti2005}.}
	
	{\it Lemma 1 \cite{Bacciotti2005}}: Consider system (\ref{eq:ftconspts}) with ${\bf f}({\bf x})={\bf f}_H({\bf x})+ {\bf f}_{NH}({\bf x})$. {Let ${\bf f}_{H}({\bf x})$ be} an $({\bf r},l_r)$-homogeneous continuous vector field such that $\x={\bf 0}$ is a locally AS equilibrium point of ${\bf \dot x}={\bf f}_{H}(\mathbf{x})$. Assume that ${\bf f}_{NH}({\bf x})$ is a continuous vector field such that $\mathbf{f}_{NH}(\mathbf{0})={\bf 0}$ and the following {\it vanishing condition} 
	$
	\lim\limits_{\epsilon\rightarrow0}\epsilon^{-(l_r+r_j)}f_{NH_j}(\Delta_{\epsilon}^{\bf r}{\bf x})=0
	$ 
	holds uniformly with respect to (w.r.t.) { ${\bf x}\in S^{m-1}$, for all $j\in \bar m$. Then, ${\bf x}={\bf 0}$ is locally AS. Further, if $l_r=0$ and all $r_j=1$, then  ${\bf x}={\bf 0}$ is locally Exponentially Stable (ES). Otherwise, if $l_r<0$, ${\bf x}={\bf 0}$ is locally FTS.} $\hfill \triangleleft$
	
	The next lemma is a direct consequence of {\it Lemma 1}, see \cite{Rio2014,Rio2017,Hong2002,Orlov2009} for other equivalent versions. 
	
	{\it Lemma 2}: Suppose that system (\ref{eq:ftconspts}) admits a homogeneous approximation of the form ${\bf \dot x}={\bf f}_{H}(\mathbf{x})$ such that the vector field ${\bf f}_H({\bf x})$ is $(\mathbf{r}, l_r)$-homogeneous and ${\bf x=0}$ is AS. Further, let ${\bf x= 0}$ be globally AS. { If $l_r=0$ and all $r_j=1$, then ${\bf x= 0}$ is globally AS and locally ES. Otherwise, if $l_r<0$,  ${\bf x= 0}$ is globally FTS.} $\hfill \triangleleft$
	
	{\it Remark 1}: The closed-loop system that arises in this paper can be written in the form of (\ref{eq:ftconspts}), { which can be separable into homogeneous and non-homogeneous terms. Hence, for some $\mathbf{f}_H(\mathbf{x})$, system (\ref{eq:ftconspts}) can be transformed to ${\mathbf {\dot x}} = \mathbf{f}_H(\mathbf{x}) +\mathbf{f}_{NH}(\mathbf{x})$, where $\mathbf{f}_{NH}(\mathbf{x}) = \mathbf{f}(\mathbf{x}) -\mathbf{f}_{H}(\mathbf{x})$. Here,} the {\it vanishing condition} is equal to $\displaystyle \lim_{\epsilon \to 0} \sup_{\mathbf{x}\in\mathcal{S}^{m-1}} \| \epsilon^{-l_r}(\Delta_\epsilon^{\mathbf{r}})^{-1} \mathbf{f}(\Delta_\epsilon^{\mathbf{r}}\mathbf{x}) - \mathbf{f}_{H}(\mathbf{x}) \| = 0$. $\hfill \triangleleft$
	
	{ Now, consider system (\ref{eq:ftconspts}) written in its input-output form:
	\begin{equation}
		\dot{\x} = \f(\x, \mathbf{u}), \quad \mathbf{y} = \mathbf{h}(\x),
		\label{eq:input-output form}
	\end{equation}
	where $\mathbf{u}\in \rea^n$ is the input, and $\mathbf{y}\in \rea^n$ is the output.
	The next lemma is a well-known result of passivity, see e.g. \cite{Ortega1998}. This is widely used to prove global asymptotic stability.
	
	{\it Lemma 3} \cite{Ortega1998}: Suppose that system (\ref{eq:input-output form}) is output strictly passive from the output $\mathbf{y} = \mathbf{h}(\x)$, and the storage function is ${\mathcal{H}}(\mathbf{x}) > 0$ for all $\mathbf{x} \neq \mathbf{0}$, and ${\mathcal{H}}(\mathbf{0}) = 0$. If (\ref{eq:input-output form}) is zero-state detectable, this is, $\mathbf{y}={\bf 0} \Rightarrow \lim\limits_{t\rightarrow \infty} \x(t) = \mathbf{0}$, then $\mathbf{x} = \mathbf{0}$ is AS. Additionally, if ${\mathcal{H}}(\mathbf{x})$ is radially unbounded, then $\mathbf{x} = \mathbf{0}$ is globally AS.}   $\hfill \triangleleft$

\section{System Dynamics and Controller Design}
	
	{Each robot is modeled as an EL system with potential energy ${^s\calU} _i (\q _i)$, and kinetic energy ${^s \mathcal{K}_i} (\q_i, \dq_i) := \hal \dq_i^{\top}\mathbf{M}_i(\q_i)\dq_i$, where $\mathbf{M}_i(\q_i) \in \rea^{n\times n}$ is the inertia matrix. The vectors $\q_i$ and $\dq_i \in \rea ^n$ are the generalized position and velocity of the $i \in \{l, r\}$ manipulator. The inputs of the system are the control law $\t _i$, and the external {\em unknown} forces $\f _i$ that are functions of time.  Following the EL equations of motion, the dynamics of each robot is}
	\begin{equation}
		{\bf M}_i (\q _i) \ddq _i + {\bf C}_i (\q _i, \dq _i) \dq _i + \nabla_{\q_i} {^s\calU} _i (\q _i) = \t _i + \f _i,
		\label{eq:manipulator_model}
	\end{equation}
	where ${\bf C}_i(\q_i,\dq_i) \in \rea^{n\times n}$ is the Coriolis and centrifugal forces matrix defined via the Christoffel symbols of the first kind. As usual in passivity-based control of bilateral teleoperation \cite{Nuno2008}, we assume the following:

	\textit{Assumption  (\textbf{A1}):} The human and the environment define passive (velocity to force) maps, i.e., there exists $\kappa_i \in \rea_{> 0}$, {such that $ {\mathcal E}_i(t):= \kappa_i - \int_{0}^{t} \dq_i(\sigma) ^{\top} \f_i(\sigma) d\sigma\ge0$, for all $t \geq 0$.} $\hfill \triangleleft$
	
	Also, we assume {boundedness of} the inertia matrix:
	
	\textit{Assumption  (\textbf{A2}):} There exist $m_{1i},m_{2i} \in \mathbb{R}_{>0}$, such that, for all $\q_i\in \rea^n$, the inertia matrix {satisfies} $m_{1i}\mathbf{I}_n\leq{\bf M}_{i}(\q_{i})\leq m_{2i}\mathbf{I}_n$; consequently, $m_{2i}^{-1}\mathbf{I}_n\leq{\bf M}_{i}(\q_{i})^{-1}\leq m_{1i}^{-1}\mathbf{I}_n$. $\hfill \triangleleft$
	
	Moreover, system \eqref{eq:manipulator_model} satisfies the following properties:
	
	\textit{Property  (\textbf{P3}):} The matrix ${\bf \dot M}_i(\q_i)-2{\bf C}_i(\q_i,\dq_i)$ is skew-symmetric. $\hfill \triangleleft$

{
	\textit{Property  (\textbf{P4}):} There exists a constant $L_{ci}>0$ such that $\|{\bf C}_{i}(\q_{i},\dq_{i})\dq_{i}\|\leq L_{ci}\|\dq_i\|^2$. $\hfill \triangleleft$

	\textit{Property  (\textbf{P5}):} For robots having only revolute joints, there exists $g_{ik}>0$, such that $|\nabla_{q_{ik}} {^s\calU}_i (\q _i)|\leq g_{ik}$ for all $\q_i\in\rea^n$, where $\nabla_{q_{ik}} {^s\calU}_i (\q _i)$ is the $k$th-element of $\nabla_{\q_i} {^s\calU}_i (\q _i)$ and $k\in\{1,...,n\}$. $\hfill \triangleleft$

Since velocities might not be available for measurement, in these cases damping cannot be directly injected using the robots velocities. Thus, we follow the standard passivity-based control methodology \cite{Ortega1998}, to design a dynamic controller that dissipates energy and that shapes the potential energy of the robots. We note that in our scenario the kinetic energy does not need to be shaped, because it already has a global minima where desired. Thus, the resulting first-order controller dynamics is given by
	\begin{equation}
		\nabla_{\dth_i} {^c\calF_i} (\dth_i) + \nabla_{\th_i} {^c\calU} (\q, \th) = \bf{0},
		\label{eq:controllers_dynamic2}
	\end{equation}
	where $\th_i, \dth_i \in \rea ^n$ are the virtual positions and velocities of the controller, and we {have} defined the vectors $\th := {\rm col}\{\th_i\}$ and $\q := {\rm col}\{\q_i\}$. Function ${^c\calU} (\q, \th)$ is the potential energy of the controller, and ${^c\mathcal{F}_i}(\dth_i)$ is a dissipation function that is designed as ${^c\mathcal{F}_i}(\dth_i) := \frac{D_{ci}}{p_\mathcal{F} + 1} \dth_i^{\top} \lceil \dth_i \rfloor ^{p_\mathcal{F}}$ with $D_{ci} \in \mathbb{R}_{\ge  0}$ and $p_\mathcal{F} \in (0,1]$. With this choice, the controller dynamics \eqref{eq:controllers_dynamic2} becomes
	\begin{equation}
	\label{newDynCon}
		\dth_i =  -  \left\lceil {1\over D_{ci}} \nabla_{\th_i} {^c\calU} (\q, \th)\right\rfloor^{1 \over p_\mathcal{F}}.
	\end{equation}
	Now, we employ Assumption {\bf A1}, to show the fact that (\ref{eq:manipulator_model}) is passive from $\t_i$ to $\dq_i$ with storage function $^s \calH(\q, \dq)$, where ${^s \calH}(\q, \dq) := \sum\limits_{i\in\{l,r\}} \Big({^s\calK}_i (\q_i, \dq_i) + {^s\calU_i} (\q_i) + {\mathcal E}_i(t)\Big)$. In fact, it holds that  $^s\dot{\calH}(\q, \dq) = \sum\limits_{i\in\{l,r\}} \dq_i^\top\t_i$. Moreover, the dynamics (\ref{eq:controllers_dynamic2}) is passive from $\dq_i$ to $\nabla_{\q}{^c\mathcal{U}}(\q, \th)$ with storage function ${^c\calU}(\q, \th)$, ---see \cite{Ortega1998}. Further, it can be also verified that it holds that $^c\dot{\mathcal{U}}(\q, \th) = \sum\limits_{i\in\{l,r\}} \Big(  \dq_i ^{\top} \nabla_{\q}{^c\mathcal{U}}(\q, \th) - D_{ci}|\dth_i|^{p_\mathcal{F}+1}\Big)$. 
	
	The first step in the controller design is to cancel-out the effects of the natural robots' potential energy. Thus, setting ${^c\calU}(\q,\th) = {^d\calU}(\q, \th) -  \sum\limits_{i\in\{l,r\}} {^s\calU_i}(\q_i)$, where ${^d\calU}(\q,\th)$ is a desired potential energy that the closed-loop system must exhibit.

Defining $\mathcal{H}(\q, \dq, \th)$ as the total energy of the robots and the controller, then it holds that 
	\begin{equation}
	\begin{aligned}
		\mathcal{H}:= &\, {^c\calU}(\q, \th) +  {^s \calH}(\q, \dq)  \\
		= &\,{^d\calU}(\q,\th) +  \sum\limits_{i\in\{l,r\}} \Big( {^s \calK_i}(\q_i, \dq_i) +{\mathcal E}_i(t)\Big).
	\end{aligned}
			\label{eq:total_energy}
	\end{equation}
Therefore, 
\begin{equation}
		\dot{\mathcal{H}} = \sum\limits_{i\in\{l,r\}} \Big( \dq_i^\top\t_i + \dq_i^\top\nabla_{\q_i}{^c\calU}(\q,\th)  -  D_{ci}|\dth_i|^{p_\mathcal{F}+1} \Big).
		\label{eq:total_energy_dot}
	\end{equation}
The energy-shaping controller $\t_i$ arises naturally, from \eqref{eq:total_energy_dot}, as
	\begin{equation}
		\t_i := -\nabla_{\q_i}{^c\mathcal{U}}(\q, \th) - \nabla_{\dq_i}{^s\mathcal{F}_i}(\dq_i),
\label{eq:energy_shaping_controller}
	\end{equation}
	where ${^s\mathcal{F}_i}(\dq_i)$ is an additional dissipation function that can be employed when velocity measurements are available. Clearly, controller \eqref{eq:energy_shaping_controller} ensures that $\dot{\mathcal{H}}(\q, \dq, \th)$ satisfies
	\begin{equation}
		\dot{\mathcal{H}}(\q, \dq, \th) = - \sum\limits_{i\in\{l,r\}} \Big(D_{ci}|\dth_i|^{p_\mathcal{F}+1} + \dq_i^\top \nabla_{\dq_i}{^s\mathcal{F}_i}(\dq_i)\Big).
		\label{total_energy_dot_final}
	\end{equation}

When velocity measurements are not available, then controller (\ref{eq:energy_shaping_controller}) becomes
\begin{equation}
		\t_i = -\nabla_{\q_i}{^c\mathcal{U}}(\q, \th).
		\label{NovelControl}
\end{equation}
	
Controllers \eqref{eq:energy_shaping_controller} or \eqref{NovelControl} with dynamics (\ref{newDynCon}) present a {\em family} of controllers rather than a single scheme. In the next section we describe the properties of functions ${^d\calU} (\q, \th)$ and ${^s\mathcal{F}_i}(\dq_i)$ in order to ensure global FT stability when there are no external forces applied in the system.
}

\section{Achieving Global Finite-Time Convergence}

The results in this section are divided in two cases, state-feedback and output feedback. 	
	
{

\subsection{State-Feedback Control}

When velocities are available, the controller is given by \eqref{eq:energy_shaping_controller} and, we set $D_{ci}=0$ ---this results in a static controller. 

Let us now define $\tq_i$ and $\dq_i$ as the state vector, for which $\tq_i := \q_i - \q_c$ with $\q_c \in \rea^n$ is a constant local-remote consensus position. Then the closed-loop system becomes
\begin{equation}
		\begin{aligned}
			\dtq_i = &\, \, \dq_i, \\
			\ddq_i = & \, \, {\bf M}_i (\tq_i + \q_c)^{-1}\Big[ \f_i(t)   - {\bf C}_i (\tq_i + \q_c, \dq_i) \dq_i \\
			& \,\,\, - \nabla_{\dq_i}{^s\mathcal{F}_i}(\dq_i)  - \nabla_{\tq_i} {^d\calU} (\tq + {\bf 1}_2\otimes \q_c, {\bf 0}) \Big].
		\end{aligned}
		\label{SF:CL}
\end{equation}

Assigning the vectors $\mathbf{r}_1 := r_1 \mathbf{1}_n$ and $\mathbf{r}_2 := r_2 \mathbf{1}_n$, where $r_1$ and $r_2 \in \rea_{>0}$, as the homogeneity weights associated with the coordinates $\tq_i$ and $\dq_i$, respectively. We can readily check that the first equation in \eqref{SF:CL} is homogeneous of degree $r_2-r_1$. Therefore, for all the closed-loop to admit a homogeneous approximation ${^s\mathcal{F}_i}(\dq_i)$ and ${^d\calU} (\tq + {\bf 1}_2\otimes \q_c, {\bf 0})$ have to admit an HA of degree $3r_2 - r_1$ and $2r_2$, respectively. Note that if $r_1>r_2$, then the homogeneity degree is negative. This, plus AS allows to conclude, via Lemma 2, that $\tq_i=\dq_i=0$ is globally FTS.

The easiest choice is to design ${^s\mathcal{F}_i}(\dq_i)$ as
\begin{equation}
	\label{dis1}
	{^s\mathcal{F}_i}(\dq_i):= \frac{D_{si}}{p_\mathcal{F} + 1} \dq_i^{\top} \lceil \dq_i \rfloor ^{p_\mathcal{F}},
\end{equation}
with $D_{ci} \in \mathbb{R}_{\ge  0}$. Setting 
\begin{equation}
	\label{pF}
p_\mathcal{F} := \frac{2r_2-r_1}{r_2},
\end{equation}
ensures that ${^s\mathcal{F}_i}(\dq_i)$ is homogeneous of degree $3r_2-r_1$, as required. Additionally, designing ${^d\calU}$ as
\begin{equation}
	\label{dPot1}
{^d\calU} = \frac{K_s}{p_\mathcal{U} + 1} (\q_l - \q_r)^\top \lceil \q_l - \q_r \rfloor ^{p_\mathcal{U}},
\end{equation}
where $K_s \in \mathbb{R}_{\ge  0}$ and
\begin{equation}
	\label{pU}
p_\mathcal{U} := \frac{2r_2-r_1}{r_1},
\end{equation}
establishes that ${^d\calU}$ is homogeneous of degree $2r_2$. This results in \textit{Controller (\textbf{C1}),} which is a simple \textit{P+d scheme} given by
	\begin{equation}
		\begin{aligned}
			\t_l &= - K_s \lceil \q_l - \q_r \rfloor ^{p_{\mathcal{U}}} - D_{sl} \lceil \dq_l\rfloor ^{p_{\mathcal{F}}} + \nabla_{\q_l} {^s\calU_l}  (\q_l),\\
			\t_r &= - K_s \lceil \q_r - \q_l \rfloor ^{p_{\mathcal{U}}} - D_{sr} \lceil \dq_r\rfloor ^{p_{\mathcal{F}}} + \nabla_{\q_r} {^s\calU_r}  (\q_r).
		\end{aligned}
		\label{cont:p+d}
	\end{equation}
	
\begin{prop}
	Consider a bilateral teleoperation system (\ref{eq:manipulator_model}) satisfying {\bf A1} and {\bf A2}. Set the energy-shaping controller (\ref{eq:energy_shaping_controller}) with ${^d\mathcal{U}}(\q,\th)$ and ${^s\mathcal{F}_i}(\dq_i)$ as in \eqref{dPot1} and \eqref{dis1}, respectively. Finally, set the homogeneity weights as:
	\begin{equation}
		2r_2>r_1>r_2>0.
	\label{eq:homogeneity_wheights}
	\end{equation}
Then
	\begin{enumerate}
		\item[A.] Position error and velocities are bounded, i.e., $\q_l - \q_r, \dq_i \in \mathcal{L}_\infty$.
		\item[B.] Additionally, if there are no external forces, i.e., $\f_i=\mathbf{0}$, the equilibrium $(\q_l - \q_r, \dq_l, \dq_r) = \mathbf{0}$ is globally FTS.
	\end{enumerate}
	\hfill $\triangleleft$
\end{prop}

{\bf Proof.} We first establish boundedness of trajectories (Conclusion A). Using \eqref{dPot1} yields the total energy \eqref{eq:total_energy} positive definite and radially unbounded with regards to (w.r.t.) $\q_l - \q_r, \dq_i$. Further, $\dot{\mathcal{H}} =  - \sum\limits_{i\in\{l,r\}}  \sum\limits_{k=1}^n D_{si} |\dot q_{ik}| ^{p_\mathcal{F}+1} \leq 0$. Therefore, $\q_l - \q_r, \dq_i$ are bounded. In fact, it also follows that 
$$
\mathcal{H}(t) + \sum\limits_{i\in\{l,r\}}\sum\limits_{k=1}^n  D_{si}\int_0^t |\dot q_{ik}(\sigma)| ^{p_\mathcal{F}+1} d\sigma \leq \mathcal{H}(0),
$$
thus, for any given $T\ge0$, by Jensen's inequality there exists $\alpha_i>0$, such that
$$
\begin{aligned}
\mathcal{H}(0) \ge & D_{si} \int_0^T |\dot q_{ik}(\sigma)|^{p_\mathcal{F}+1}d\sigma \ge \alpha_i \left|\int_0^T \dot q_{ik}(\sigma) d\sigma \right|^{p_\mathcal{F}+1}\\
\ge & \alpha_i \left| q_{ik}(T) - q_{ik}(0) \right|^{p_\mathcal{F}+1}, \quad  \forall k\in\{1,...,n\}.
\end{aligned}
$$
Hence, the positions $\q_i$ are also bounded. Thus, $\tq_i\in \mathcal{L}_\infty$, for any constant $\q_c\in\rea^n$.

We now set $\f_i={\bf 0}$ to prove Conclusion B. In this case, the unknown external forces $\f_i(t)$ vanish from the closed-loop \eqref{SF:CL} and also ${\mathcal E}_i(t)$ vanishes from the total energy \eqref{eq:total_energy}. Thus, LaSalle's invariance principle ensures that $(\q_l - \q_r, \dq_l, \dq_r) = \mathbf{0}$ is globally AS. Hence, there exists $\q_c\in\rea^n$ such that $(\tq_l,  \tq_r, \dq_l, \dq_r) = \mathbf{0}$ is globally AS. In this case, $\dq_i$ is the {\em detectable} output of Lemma 3.

In order to apply Lemma 2, it rests to prove that \eqref{SF:CL} admits a homogeneous approximation of negative degree. For, note that we can write $\ddq_i = {\bf f}_{Hi}(\tq_l,\tq_r, \dq_i) + {\bf f}_{NHi}(\tq_l,\tq_r, \dq_i)$, where
\begin{equation}
	\label{fh}
{\bf f}_{Hi}:= - {\bf M}_i (\q_c)^{-1}\Big[ K_s \lceil \tq_i - \tq_j \rfloor ^{p_{\mathcal{U}}} +  D_{si} \lceil \dq_i\rfloor ^{p_{\mathcal{F}}}\Big],	
\end{equation}
with $j\in\{r,l\}$ and
$$
\begin{aligned}
	{\bf f}_{NHi} := &- {\bf M}_i (\tq_i + \q_c)^{-1}{\bf C}_i (\tq_i + \q_c, \dq_i) \dq_i \\ &- \left[{\bf M}_i (\tq_i + \q_c)^{-1} - {\bf M}_i (\q_c)^{-1} \right]  K_s \lceil \tq_i - \tq_j \rfloor ^{p_{\mathcal{U}}} \\  &   - \left[{\bf M}_i (\tq_i + \q_c)^{-1} - {\bf M}_i (\q_c)^{-1} \right]  D_{si} \lceil \dq_i\rfloor ^{p_{\mathcal{F}}}.
\end{aligned}
$$
From Assumption {\bf A2} it holds that the system $\dtq_i =\dq_i$, $\ddq_i = {\bf f}_{Hi}$ is homogeneous of degree $r_2-r_1$ and since $r_1>r_2$ the degree is negative. Additionally, using the function 
\begin{equation}
	\label{V}
	V = {^d\calU} + {1\over 2}\dq_l^\top {\bf M}_l (\q_c) \dq_l + {1\over 2}\dq_r^\top {\bf M}_r (\q_c) \dq_r, 
\end{equation}
where ${^d\calU}$ is given in \eqref{dPot1}, it can be proved that the origin of such system is AS. It only rests to prove that 
$$
\lim_{\epsilon \to 0} {1\over \epsilon^{2r_2-r_1}} \left\| {\bf f}_{NHi}(\epsilon^{r_1}\tq_l, \epsilon^{r_1}\tq_r, \epsilon^{r_2}\dq_i) \right\| = 0.
$$

The first term in ${\bf f}_{NHi}$ vanishes due to Property {\bf P4} with the fact that
\begin{equation*}
	\begin{aligned}
		\lim_{\epsilon \to 0}& { \left\| {\bf M}_i(\epsilon^{r_1}\tq_i + \q_c)^{-1}\mathbf{C}_i(\epsilon^{r_1}\tq_i + \q_c, \epsilon^{r_2}\dq_i)\epsilon^{r_2}\dq_i \right\| \over \epsilon^{2r_2-r_1} }\\
		&\leq \lim_{\epsilon \to 0} {m_{2i} L_{ci} \left\| \epsilon^{r_2} \dq_i \right\|^{2}\over \epsilon^{2r_2-r_1}} = m_{2i} L_{ci} \left\| \dq_i \right\|^{2}\lim_{\epsilon \to 0}\epsilon^{r_1} = 0.
	\end{aligned}
\end{equation*}

The last two terms in ${\bf f}_{NHi}$ vanish because $\lceil \tq_i - \tq_j \rfloor ^{p_{\mathcal{U}}}$ and $\lceil \dq_i\rfloor ^{p_{\mathcal{F}}}$ are already homogeneous of degree $2r_2 -r_1$ plus the fact that $\lim\limits_{\epsilon \to 0}  \left[{\bf M}_i (\epsilon^{r_1}\tq_i + \q_c)^{-1} - {\bf M}_i (\q_c)^{-1} \right] = {\bf 0}$. This completes the proof. 
\hfill $\square$

\subsection{Output-Feedback Control}

When velocities are not available, then ${^s\mathcal{F}_i}(\dq_i)=0$ because damping cannot be directly injected in the robot dynamics. Instead, by setting $D_{ci}>0$, damping is injected through the controller dynamics \eqref{newDynCon}. Further, additional to the given homogeneity properties that the desired potential energy ${^d\calU}$ has to observe, it also has to satisfy the typical detectability condition in passivity-based control ---see Lemma 3. To achieve this, let us define another state variable that accounts for the robot-controller interconnection as $\tt_i:=\th_i -\q_i$, and set ${^d\calU}$ as
\begin{equation}
	\label{dPot2}
	\begin{aligned}
{^d\calU} = &\frac{1}{p_\mathcal{U} + 1} K_s (\q_l - \q_r)^\top \lceil \q_l - \q_r \rfloor^{p_\mathcal{U}} \\ & +\frac{1}{p_\mathcal{U} +1}  \left[ K_{cl} \tt_l^\top \lceil \tt_l \rfloor^{p_\mathcal{U}} + K_{cr} \tt_r^\top \lceil \tt_r \rfloor^{p_\mathcal{U}} \right],
	\end{aligned}
\end{equation}
where $K_s, K_{ci}>0$ and $p_\mathcal{U}$ is defined in \eqref{pU}.  These choices result in \textit{Controller (\textbf{C2}),} which is a dynamic \textit{P+d controller without velocity measurements} given by
\begin{equation}
		\begin{aligned}
			\t_l = &- K_s \lceil \q_l - \q_r \rfloor ^{p_{\mathcal{U}}} + K_{cl} \lceil \tt_l\rfloor ^{p_{\mathcal{U}}} + \nabla_{\q_l} {^s\calU_l}  (\q_l),\\
			\t_r = &- K_s \lceil \q_r - \q_l \rfloor ^{p_{\mathcal{U}}} + K_{cr} \lceil \tt_r\rfloor ^{p_{\mathcal{U}}} + \nabla_{\q_r} {^s\calU_r}  (\q_r), \\
		\dth_i =  & -  \Big({K_{ci}\over D_{ci}}\Big)^{1 \over p_\mathcal{F}} \lceil \tt_i \rfloor^{r_2 \over r_1} .
		\end{aligned}
\label{cont:p+d without}
\end{equation}

\begin{prop}
Conclusions A and B of Proposition 1 hold if the bilateral teleoperation system (\ref{eq:manipulator_model}) is controlled by (\ref{cont:p+d without}), provided that \eqref{eq:homogeneity_wheights} holds. 	\hfill $\triangleleft$
\end{prop}

{\bf Proof.} 
The resulting closed-loop dynamics are
\begin{equation}
		\begin{aligned}
			\dtq_i = &\, \, \dq_i, \\
			\ddq_i = & \, \, {\bf M}_i (\tq_i + \q_c)^{-1}\Big[ \f_i(t)   - {\bf C}_i (\tq_i + \q_c, \dq_i) \dq_i \\
			& \,\,\,  - K_s \lceil \q_i - \q_j \rfloor ^{p_{\mathcal{U}}} +  K_{ci} \lceil \tt_i\rfloor ^{p_{\mathcal{U}}}   \Big],\\
			\dtt_i =  & -  \Big({K_{ci}\over D_{ci}}\Big)^{1 \over p_\mathcal{F}} \lceil \tt_i \rfloor^{r_2 \over r_1} - \dq_i.
		\end{aligned}
		\label{SF:CL2}
\end{equation}

The boundedness of trajectories (Conclusion A) proof follows replacing \eqref{dPot2} in \eqref{eq:total_energy} to obtain a positive definite and radially unbounded function w.r.t. $\q_l - \q_r, \dq_i, \tt_i$. In this case \eqref{total_energy_dot_final} satisfies $
\dot{\mathcal{H}} =  - \sum\limits_{i\in\{l,r\}}  \sum\limits_{k=1}^n D_{ci} |\dot \theta_{ik}| ^{p_\mathcal{F}+1} \leq 0$. Therefore, $\q_l - \q_r, \dq_i, \tt_i \in \mathcal{L}_\infty$. We can also establish that there exists $\beta_i>0$ such that $\mathcal{H}(0) \ge \beta_i \left| \theta_{ik}(t) - \theta_{ik}(0) \right|^{p_\mathcal{F}+1}$. Hence, the virtual positions $\theta_i\in \mathcal{L}_\infty$. This last, and since $\tt_i \in \mathcal{L}_\infty$, then $\q_i\in \mathcal{L}_\infty$ and thus, $\tq_i\in \mathcal{L}_\infty$. This completes the proof of Conclusion A.

To establish Conclusion B, set $\f_i={\bf 0}$, thus the unknown external forces $\f_i(t)$ do not appear in the closed-loop \eqref{SF:CL2} and also ${\mathcal E}_i(t)$ vanishes from the total energy \eqref{eq:total_energy}. Here, $\dth_i$ is the detectable output and the detectability condition of Lemma 3 holds because $\dth_i={\bf 0}$ implies, from the last equation in \eqref{cont:p+d without}, that $\tt_i={\bf 0}$. These in turn ensure that $\q_i=\th_i={\rm constant}$. Hence $\ddq_i=\dq_i={\bf 0}$. Thus, LaSalle's invariance principle ensures that $(\q_l - \q_r, \dq_l, \dq_r, \tt_l, \tt_r) = \mathbf{0}$ is globally AS. Hence, there exists $\q_c\in\rea^n$ such that $(\tq_l,  \tq_r, \dq_l, \dq_r, \tt_l,\tt_r) = \mathbf{0}$ is globally AS.

The rest of the proof shows that \eqref{SF:CL2} admits a homogeneous approximation of negative degree. In this case, we assign $\mathbf{r}_3 := r_1 \mathbf{1}_n$ to be the homogeneity weight of $\tt_i$.

At this point we want to underscore that the first and the last equations in \eqref{SF:CL2} are already homogeneous with degree $r_2-r_1$. For the second equation we write it as $\ddq_i = {\bf f}_{Hi}(\tq_l,\tq_r, \dq_i,\tt_i) + {\bf f}_{NHi}(\tq_l,\tq_r, \dq_i, \tt_i)$, where
$$
{\bf f}_{Hi}:= - {\bf M}_i (\q_c)^{-1}\Big[ K_s \lceil \tq_i - \tq_j \rfloor ^{p_{\mathcal{U}}} -  K_{ci} \lceil \tt_i\rfloor ^{p_{\mathcal{U}}}\Big],
$$
with $ j\in\{r,l\}$. Moreover,
$$
\begin{aligned}
	{\bf f}_{NHi} := &- {\bf M}_i (\tq_i + \q_c)^{-1}{\bf C}_i (\tq_i + \q_c, \dq_i) \dq_i \\ &- \left[{\bf M}_i (\tq_i + \q_c)^{-1} - {\bf M}_i (\q_c)^{-1} \right]  K_s \lceil \tq_i - \tq_j \rfloor ^{p_{\mathcal{U}}} \\  &   + \left[{\bf M}_i (\tq_i + \q_c)^{-1} - {\bf M}_i (\q_c)^{-1} \right]  K_{ci} \lceil \tt_i\rfloor ^{p_{\mathcal{U}}}.
\end{aligned}
$$
From Assumption {\bf A2} it holds that $\ddq_i = {\bf f}_{Hi}$ is also homogeneous of degree $r_2-r_1$. Additionally, using the function $V$ defined in \eqref{V}, with ${^d\calU}$ given by \eqref{dPot2}, we can establish that the origin of the homogeneous part of the closed-loop system is AS. The proof that the non homogeneous terms vanish can be easily established following {\em verbatim} the proof of Proposition 1. \hfill $\square$

\subsection{Bounded Controllers}
	
In practical applications, actuators are prone to saturation, which degrades overall performance and increases the risk of thermal and mechanical failure. Here we design bounded controllers that avoid saturation under the next assumption. 

\textit{Assumption  ({\bf A3}):} For each $i\in\{l,r\}$ and $k \in \{1,...,n\}$, there exists $\bar \tau_{ik}> 0$ such that $|\tau_{ik}| \leq \bar \tau_{ik}$, where $\bar \tau_{ik}$ is the {\em known} upper-bound of each physical robot actuator. Furthermore, each actuator is capable of lifting its own link weight, thus $\bar \tau_{ik} > g_{ik}$.  $\hfill \triangleleft$

Two different controllers are designed, one that is state-feedback and the other that is output-feedback. 

The {\bf bounded state-feedback} scheme relies on velocity measurements and thus $D_{ci}=0$. ${^d\calU}$ is designed as
\begin{equation}
	\label{dPot3}
{^d\calU} = K_s \sum\limits_{k=1}^n s\big(q_{lk} - q_{rk}, \delta_\calU, p_\calU\big),
\end{equation}
where $K_s \in \mathbb{R}_{\ge  0}$, $\delta_\calU>0$, $p_\calU$ is defined in \eqref{pU} and $s\big(\cdot, \cdot, \cdot\big)$ is given in \eqref{sfunc}. The dissipation function ${^s\mathcal{F}_i}(\dq_i)$ is now set as
\begin{equation}
	\label{dis2}
	{^s\mathcal{F}_i}(\dq_i):= D_{si} \sum\limits_{k=1}^n s\big(\dot q_{ik}, \delta_\calF, p_\calF\big),
\end{equation}
with $D_{si} >0$, $\delta_\calF>0$ and $p_\calF$ is defined in \eqref{pF}. In this case, the \textit{Controller (\textbf{C3}),} is a simple \textit{bounded P+d scheme}, which is the bounded version of \eqref{cont:p+d} and that is
	\begin{equation}
		\begin{aligned}
			\t_i = &- K_s \,{\rm sat}_{\delta_\calU}\left(\lceil \q_i - \q_j \rfloor ^{p_{\mathcal{U}}}\right) - D_{si} \,{\rm sat}_{\delta_\calF}\left(\lceil \dq_i\rfloor ^{p_{\mathcal{F}}}\right) \\
			& + \nabla_{\q_i} {^s\calU_i}  (\q_i), \qquad\qquad j\in\{r,l\},\\
		\end{aligned}
		\label{pdBounded}
	\end{equation}

\begin{prop}
Conclusions A and B of Proposition 1 hold if the bilateral teleoperation system (\ref{eq:manipulator_model}) is controlled by (\ref{pdBounded}) provided that \eqref{eq:homogeneity_wheights} is satisfied. Additionally, the actuators do not saturate if Assumption {\bf A3} and condition
\begin{equation}
	\label{scond}
	K_s \delta_\calU + D_{si} \delta_\calF < \bar \tau_{ik} - g_{ik},
\end{equation}  	
also hold, thus $|\tau_{ik}| < \bar \tau_{ik}$. \hfill $\triangleleft$
\end{prop}
	
{\bf Proof.} The proof follows {\em verbatim} the proof of Proposition~1 and thus only the main key differences are given.  

Boundedness of trajectories follows replacing \eqref{dPot3} in the total energy \eqref{eq:total_energy} to obtain a positive definite and radially unbounded function w.r.t. $\q_l - \q_r, \dq_i$. In this case $\dot{\mathcal{H}}$ satisfies
$$
\dot{\mathcal{H}} =  - \sum\limits_{i\in\{l,r\}}  \sum\limits_{k=1}^n D_{si} \left\{\begin{matrix} |\dot q_{ik}|^{p_\calF+1}& \text{if} & | \dot q_{ik} | < \delta_\calF, \\
			\delta_\calF^{p_\calF}|\dot q_{ik}| & \text{if} & | \dot q_{ik} | \geq \delta_\calF.
			\end{matrix}\right.
$$
Thus, $\dot{\mathcal{H}}\leq 0$. 

The closed-loop can be written as $\dtq_i=\dq_i$, $\ddq_i = {\bf f}_{Hi}(\tq_l,\tq_r, \dq_i) + {\bf f}_{NHi}(\tq_l,\tq_r, \dq_i)$, with ${\bf f}_{Hi}(\tq_l,\tq_r, \dq_i)$ the same as in \eqref{fh}. The proof that the origin is AS for the homogenous part follows directly. For the homogenous approximation we have to establish a vanishing condition on ${\bf f}_{NHi}(\tq_l,\tq_r, \dq_i)$, given by
$$
\begin{aligned}
	&{\bf f}_{NHi} = - {\bf M}_i (\tq_i + \q_c)^{-1}{\bf C}_i (\tq_i + \q_c, \dq_i) \dq_i \\ & \,\,\,\,- K_s\left[{\bf M}_i (\tq_i + \q_c)^{-1} - {\bf M}_i (\q_c)^{-1}\right]{\rm sat}_{\delta_\calU}\left(\lceil \tq_i - \tq_j \rfloor ^{p_{\mathcal{U}}}\right) \\
	& \,\,\,\, - D_{si} \left[ {\bf M}_i (\tq_i + \q_c)^{-1} - {\bf M}_i (\q_c)^{-1}\right] {\rm sat}_{\delta_\calF}\left(\lceil \dq_i\rfloor ^{p_{\mathcal{F}}}\right) \\
	& \,\,\,\, + K_s{\bf M}_i (\q_c)^{-1}\Big[\lceil \tq_i - \tq_j \rfloor ^{p_{\mathcal{U}}} - {\rm sat}_{\delta_\calU}\left(\lceil \tq_i - \tq_j \rfloor ^{p_{\mathcal{U}}}\right)\Big]\\
	& \,\,\,\, + D_{si} {\bf M}_i (\q_c)^{-1}\Big[\lceil \dq_i\rfloor ^{p_{\mathcal{F}}} -{\rm sat}_{\delta_\calF}\left(\lceil \dq_i\rfloor ^{p_{\mathcal{F}}}\right) \Big].
\end{aligned}
$$
As in the proof of Proposition 1, the first three terms on the right hand side of ${\bf f}_{NHi}$ vanish as $\epsilon\to 0$. Consider now the last term of ${\bf f}_{NHi}$ and note that
$$
\lim_{\epsilon \to 0} {\lceil \epsilon^{r_2}\dq_i \rfloor ^{p_{\mathcal{F}}} - {\rm sat}_{\delta_\calF}\left(\lceil \epsilon^{r_2}\dq_i \rfloor ^{p_{\mathcal{F}}}\right) \over \epsilon^{2r_2 - r_1}} = 0,
$$
because ${\rm sat}_{\delta_\calF}\left(\lceil \epsilon^{r_2}\dq_i \rfloor ^{p_{\mathcal{F}}}\right)= \lceil \epsilon^{r_2}\dq_i \rfloor ^{p_{\mathcal{F}}}$ when $\epsilon^{r_2}\dq_i\in B_{\delta_\calF}$. Note that, as $\epsilon \to 0$, $\epsilon^{r_2}\dq_i\to 0$ because $\dq_i\in\mathcal{L}_\infty$. Hence, as $\epsilon \to 0$, the term $\epsilon^{r_2}\dq_i$ always reaches the ball $B_{\delta_\calF}$. The fourth term in the right hand side of ${\bf f}_{NHi}$ vanishes using the same argument. Thus, the closed-loop accepts a homogeneous approximation of negative degree $r_2-r_1$. 

Finally, for each $k\in\{1,...,n\}$, the absolute value of each element of the controller \eqref{pdBounded} is
$$
\begin{aligned}
|\tau_{ik}| = &|- K_s \,{\rm sat}_{\delta_\calU}\left(\lceil q_{ik} - q_{jk} \rfloor ^{p_{\mathcal{U}}}\right) - D_{si} \,{\rm sat}_{\delta_\calF}\left(\lceil \dot q_{ik}\rfloor ^{p_{\mathcal{F}}}\right) \\
			& + \nabla_{q_{ik}} {^s\calU_i}  (\q_i)|,	
\end{aligned}
$$
thus
$$
|\tau_{ik}| \leq K_s |{\rm sat}_{\delta_\calU}\left(\lceil q_{ik} - q_{jk} \rfloor ^{p_{\mathcal{U}}}\right)| + D_{si} |{\rm sat}_{\delta_\calF}\left(\lceil \dot q_{ik}\rfloor ^{p_{\mathcal{F}}}\right)| + g_{ik},
$$
where we have used Property {\bf P5} to obtain this bound. Therefore, from the saturation functions we also have that
$$
|\tau_{ik}| \leq K_s \delta_\calU + D_{si} \delta_\calF + g_{ik}.
$$
Setting $K_s \delta_\calU + D_{si} \delta_\calF + g_{ik}< \bar \tau_{ik}$ ---which is the same as \eqref{scond}--- ensures that $|\tau_{ik}| < \bar \tau_{ik}$. This finishes the proof. \hfill $\square$

The last controller that we design in this work is a {\bf bounded output-feedback} scheme. This algorithm arises setting ${^s\mathcal{F}_i}(\dq_i)=0$, $D_{ci}>0$ ---in \eqref{newDynCon}--- and ${^d\calU}$ as
\begin{equation}
	\label{dPot4}
	\begin{aligned}
{^d\calU} = & \sum\limits_{k=1}^n K_s\,s\big(q_{lk} - q_{rk}, \delta_\calU, p_\calU\big)\\
&+\sum\limits_{k=1}^n \left[  K_{cl}\,s\big({\tilde \theta}_{lk}, \delta_\calF, p_\calU\big) + K_{cr}\,s\big({\tilde \theta}_{rk}, \delta_\calF, p_\calU\big) \right] 
\end{aligned}
\end{equation}
where $K_s, K_{ci}>0$ and $p_\mathcal{U}$ is defined in \eqref{pU}. Thus, the controllers are
\begin{equation}
		\begin{aligned}
			\t_i = &- K_s \,{\rm sat}_{\delta_\calU}\left(\lceil \q_i - \q_j \rfloor ^{p_{\mathcal{U}}} \right) + K_{ci} \,{\rm sat}_{\delta_\calF}\big(\lceil \tt_i\rfloor ^{p_{\mathcal{U}}}\big) \\
			& + \nabla_{\q_i} {^s\calU_i}  (\q_i), \qquad\qquad j\in\{r,l\}\\
		\dth_i =  & -  \Big({K_{ci}\over D_{ci}}\Big)^{1 \over p_\mathcal{F}}  \,{\rm sat}_{\delta_\calF}\left(\lceil \tt_i\rfloor ^{r_2\over r_1}\right).
		\end{aligned}
\label{boundedpdwithout}
\end{equation}

The next proposition is reported without proof for sake of space and because it can be derived mimicking the steps in the previous proofs.

\begin{prop}
Conclusions A and B of Proposition 1 hold if the bilateral teleoperation system (\ref{eq:manipulator_model}) is controlled by (\ref{boundedpdwithout}) provided that \eqref{eq:homogeneity_wheights} is satisfied. Additionally, the actuators do not saturate if Assumption {\bf A3} and condition
\begin{equation}
	\label{scond2}
	K_s \delta_\calU + K_{ci} \delta_\calF < \bar \tau_{ik} - g_{ik},
\end{equation}  	
also hold, thus $|\tau_{ik}| < \bar \tau_{ik}$. \hfill $\triangleleft$
\end{prop}

\section{Additional Remarks}

{\it Remark 2}. For simplicity and clarity of exposition, all the designed controllers use scalar gains. However, they can also be defined as positive diagonal matrices, where each diagonal element corresponds to the gain of one degree of freedom of the robots. The proofs remain essentially unchanged. $\hfill \triangleleft$ 
	
{\it Remark 3}. The stability proof requires that $r_1 > r_2$, resulting in a negative-degree HA in the closed-loops. However, setting $r_1 = r_2$ yields a zero-degree HA, from which global asymptotic stability and local exponential stability can be concluded (see {\it Lemma~2}). $\hfill \triangleleft$
	
{\it Remark 4}. The controller dynamics is first-order because we set ${^c\calK_i}(\dth_i) := 0$. Nevertheless, it is also possible to achieve FT convergence with second-order controllers by setting ${^c\calK}i(\dth_i) := \dth_i^{\top} \mathbf{M}{ci}, \dth_i$, with $\mathbf{M}{ci}=\mathbf{M}{ci}^\top \in \rea^{n\times n}$ a positive definite matrix. In this case, the proofs only require including the controller's kinetic energy in the total energy, and they remain unchanged. $\hfill \triangleleft$ 
	
{\it Remark 5}. Energy-shaping controllers provide a clear physical interpretation of the control action. For {\bf C1}--{\bf C4}, the proportional terms can be understood as nonlinear springs that help reduce the error when disturbances are present. The damping-injection terms reduce overshoot and oscillations, but if these gains are incorrectly tuned, they may amplify sensor noise; in such cases, the use of controllers {\bf C2} and {\bf C4} is advised, as they do not rely on velocity measurements. $\hfill \triangleleft$ 

{\it Remark 6}. Setting higher control gains leads to faster and more accurate responses, but requires larger control efforts. The bounded schemes {\bf C3} and {\bf C4} achieve the control objective while respecting the force limits of the actuators. $\hfill \triangleleft$

{\it Remark 7}. Setting $2r_2=r_1$ leads to $p_\mathcal{U} = p_\mathcal{F} = 0$, which results in discontinuous controllers. While these controllers are more robust, they are prone to chattering; thus, the setting $2r_2=r_1$ should be avoided. In contrast, setting $r_2=r_1$ leads to $p_\mathcal{U} = p_\mathcal{F}= 1$, which returns linear controllers. These provide smoother responses at the cost of weaker robustness properties and the loss of FT convergence.  $\hfill \triangleleft$ 

{\it Remark 8}. The only model-dependent term in {\bf C1}--{\bf C4} is the potential energy cancellation term. If model uncertainty arises, the FT convergence property is lost. However, these schemes are more robust than their linear counterparts, achieving smaller steady-state errors. Naturally, a sliding mode scheme performs better, at the cost of possible noise amplification and chattering.  $\hfill \triangleleft$

{\it Remark 9}. We have reported only four different controllers. However, the family of energy-shaping schemes in \eqref{eq:energy_shaping_controller} or in     \eqref{NovelControl} accepts several more designs, as long as the detectability condition plus the homogeneous approximation of negative degree hold. $\hfill \triangleleft$

}

\section{Simulations and Experiments}
\subsection{Simulation Results}
In this section, we compare the controller \textbf{C1} against the Terminal-Sliding-Mode-based Controller (\textbf{TSMC}) from \cite{Nguyen2021}. The local and remote are the same 2-DoF robots from \cite{Nguyen2021}, with $M_i = [1.8, 1.6]$ kg, $l_i = [0.8,0.6]$ m, $l_{ci} = [0.4,0.3]$ m and $I_i = [0.096, 0.048]$ kg$\cdot$m\textsuperscript{2}. The control gains for \textbf{C1} are: $K_s = 6$, $D_s = 8$, $r_1 = 1.5$ and $r_2 = 1$. The control gains for \textbf{TSMC} are: $\lambda = 3$, $\Lambda = 10$, $K_i = 3\cdot \mathbf{I}_2$, $\Gamma_i = 4\cdot \mathbf{I}_2$, $\Psi_i = \mathbf{I}_2$, $\delta = \rho = 0.5$ and $\epsilon = 0.001$. The simulations use the Euler’s discretization with fixed step $T_s = 0.1$ ms. 

\begin{figure}[h]
	\centering
	\includegraphics[trim={1.35cm 4.6cm 2.8cm 2.8cm},clip,width=1\columnwidth]{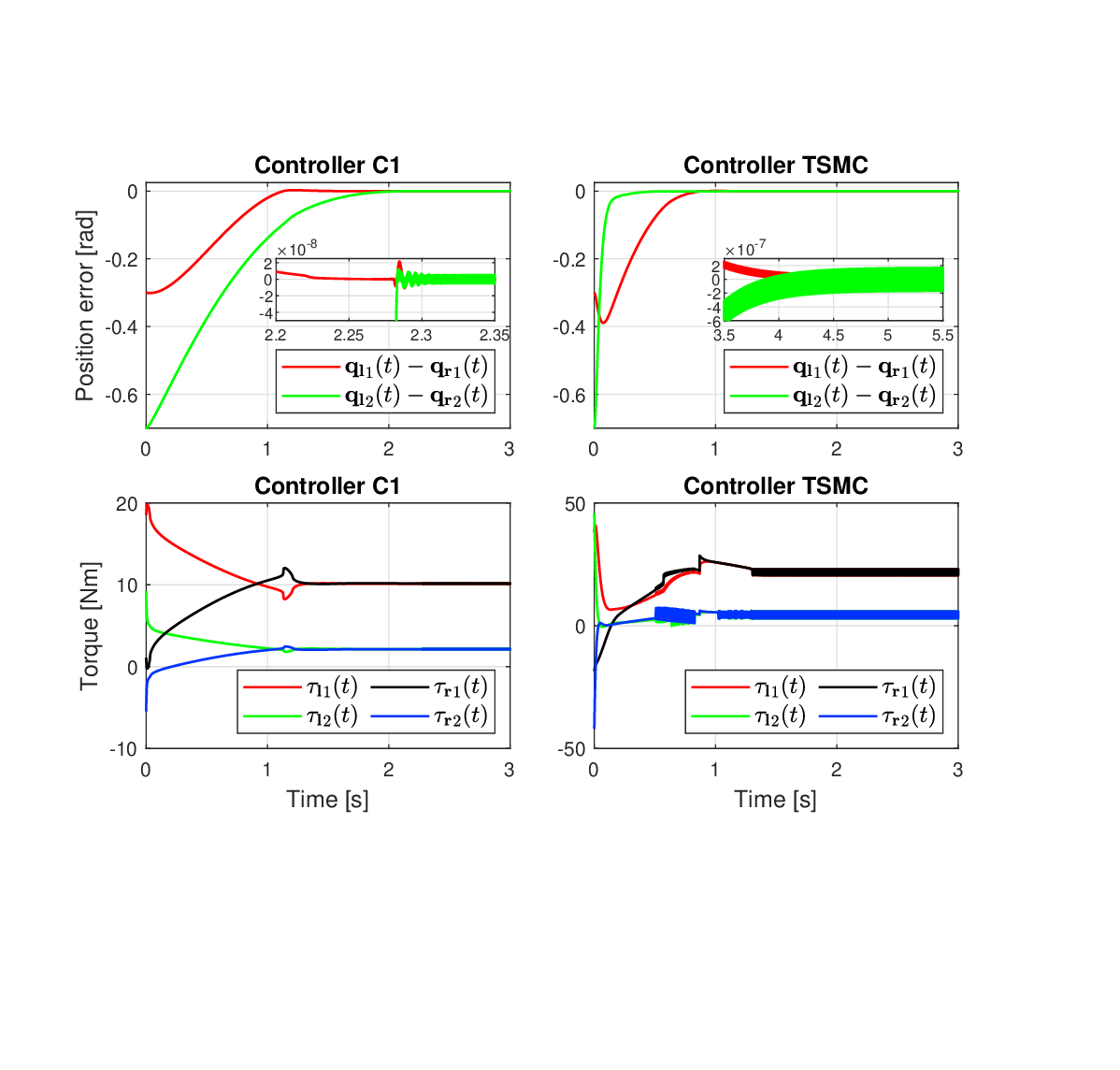}
	\caption{Position error and torques of controllers {\bf C1} and {\bf TSMC}.}
	\label{fig:Comparison_TSMC}
\end{figure}

In the beginning, the robots parted from rest in their initial positions $\q_l(0) = [1,-0.4]^\top$ rad and $\q_r(0) = [1.3, 0.3]^\top$ rad, and they were set in free motion. Figure \ref{fig:Comparison_TSMC} shows the controllers achieving FT convergence of errors to zero at $t = 2.3$ s for \textbf{C1} and $t=5$ s for \textbf{TSMC}. Nevertheless, \textbf{TSMC} presented an aggressive control law as chattering is evident, whereas \textbf{C1} was able to attenuate this effect.

\subsection{Experimental Results}	
Experimental results are presented to validate the controllers of Section V. We used a 6-DoF teleoperation system comprising a local Kinova\textsuperscript{\textregistered} Gen3 Ultra-lightweight (\textbf{KG3}) and a remote Kinova\textsuperscript{\textregistered} Gen3 Lite (\textbf{LT3}) robot. The first experiment tested the robustness of the FT controllers \textbf{C1}--\textbf{C4} against their asymptotic equivalent when they are in free motion. As it is common in practical implementations, modeling errors and static friction became a source of uncertainty, among others.

Due to their differences in model, distinct control gains were used for each robot, see \textit{Remark~2}. Also, individual gains were set for joints 1--3 of \textbf{LT3} to account for varying motor sizes. The control gains of \textbf{C1}--\textbf{C4} for \textbf{KG3} were: \( K_{s_l} = 15 \), \( K_{c_l} = 2 \), \( D_{s_l} = 2 \), \( D_{c_l} = 20 \), \( \delta_{s_l} = 0.2 \) and \( \delta_{D_l} = 0.5 \). The control gains for \textbf{LT3} were: \( K_{sr} = [4,17,4,1.5,1.5,1.5] \), \( K_{cr} = 1 \), \( D_{sr} = 0.5 \), \( D_{cr} = 20 \), \( \delta_{sr} = 0.2 \) and \( \delta_{Dr} = 0.5 \). The FT versions of \textbf{C1}--\textbf{C4} used \( r_1 = 1.5 \) and \( r_2 = 1.1 \), whereas their asymptotic versions used \( r_1 = r_2 = 1 \). The controllers were implemented using the Kinova\textsuperscript{\textregistered} Kortex\textsuperscript{TM} API, enabling low-level torque control at 1\,ms intervals and access to built-in position and velocity measurements. Communication between robots was established via UDP-IP over a LAN connection.

The experiments started with the robots parting from rest with different initial positions, for \textbf{KG3} was: \( \q_l(0) = [\frac{\pi}{9}, \frac{\pi}{6}, \frac{-2\pi}{9}, \frac{-\pi}{6}, \frac{\pi}{18}, 0]^\top \) rad, and for \textbf{LT3} was: \( \q_r(0) = [\frac{-\pi}{12}, \frac{\pi}{12}, \frac{5\pi}{36}, \frac{-1\pi}{180}, \frac{2\pi}{9}, \frac{-\pi}{3}]\top \) rad, with \( \th_i(0) = \q_i(0) \). Figure~\ref{fig:Comparison} presents the norm of the position error between robots as the evaluation metric. The controller was activated in $t=1$ s, then, both robots tried to reach a consensus position. In general, the FT controllers significantly reduced the error and presented faster transient responses.

\begin{figure}[h]
	\centering
	\includegraphics[trim={1.7cm 5.6cm 2.8cm 3.75cm},clip,width=1\columnwidth]{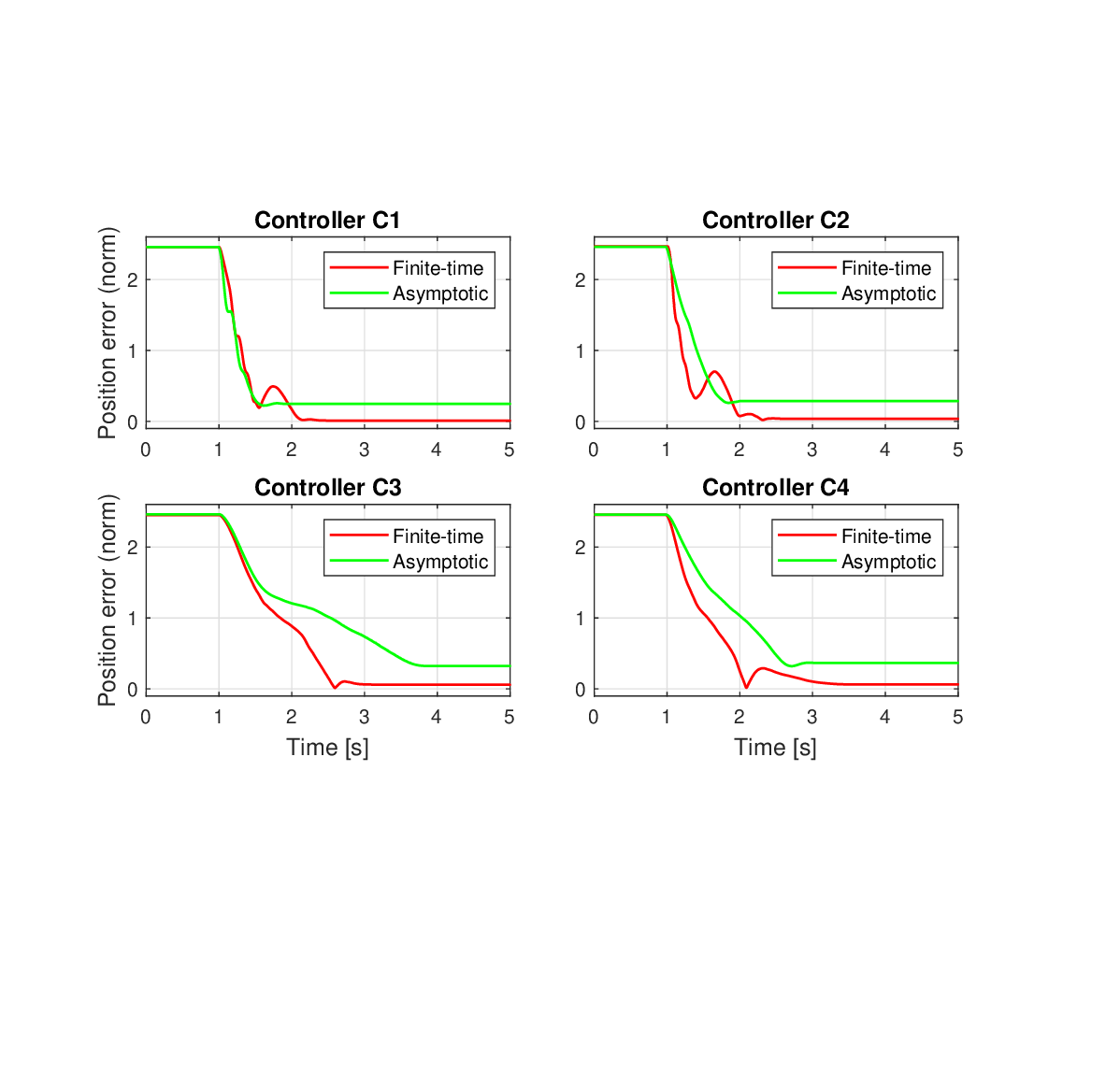}
	\caption{Position error comparison between controller {\bf C1}--{\bf C4} in their finite-time and asymptotic version.}
	\label{fig:Comparison}
\end{figure}

The second experiment tested the tracking capabilities of the saturated controller \textbf{C4}. For this, external forces were applied to each joint of \textbf{LT3} to change their position. Figure \ref{fig:Saturated} shows how \textbf{KG3} successfully tracked the position of \textbf{LT3} with minimal steady state error. To show the saturated components in \textbf{C4}, there were introduced forces that maxed out the controller in the third joints at $t = 21$ s. When these forces disappeared, the robots returned to a consensus position again.
	
\begin{figure}[h]
	\centering
	\includegraphics[trim={1.35cm 2.7cm 2.7cm 2.3cm},clip,width=1\columnwidth]{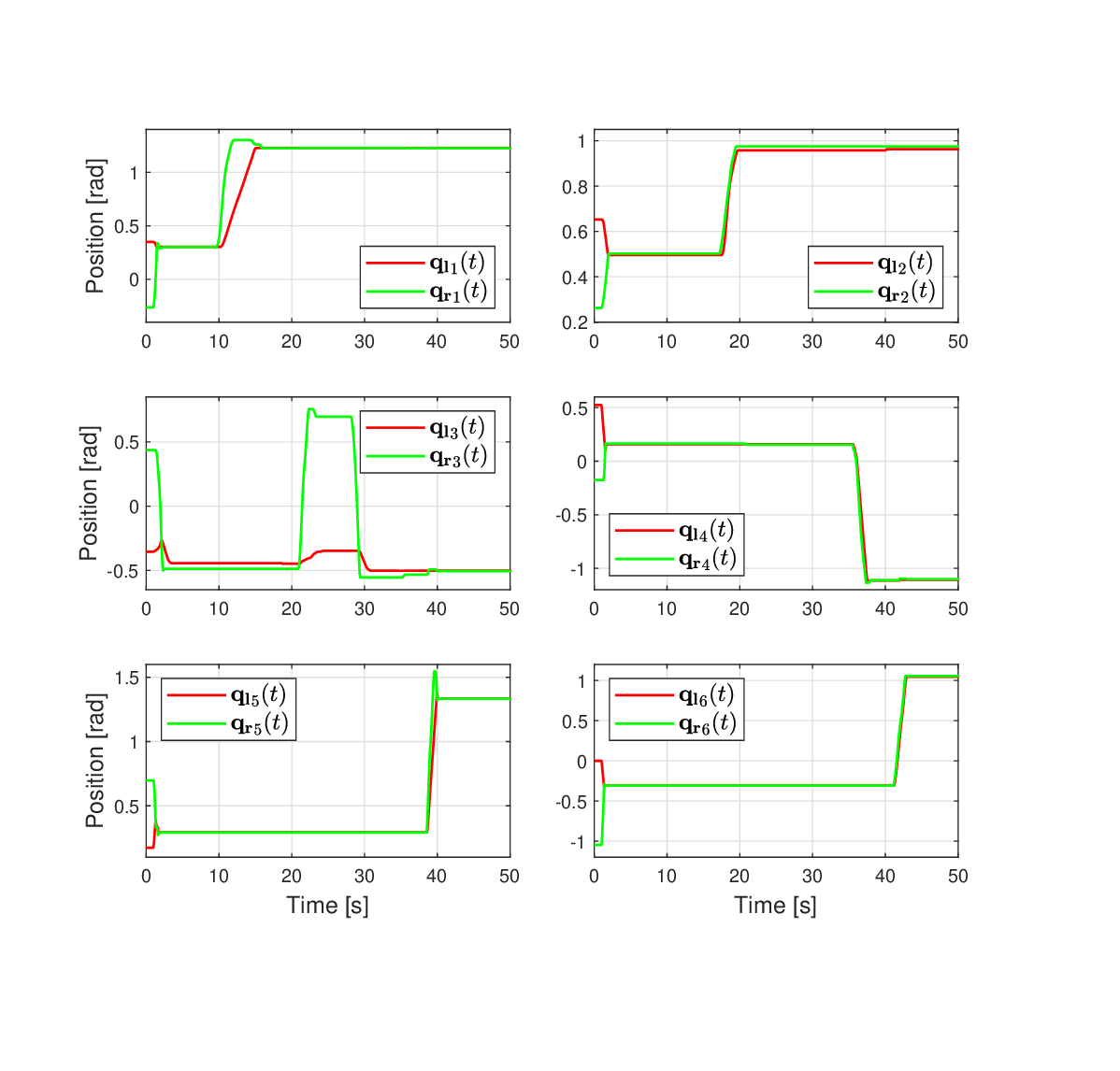}
	\caption{Position tracking of controller {\bf C4} under external forces.}
	\label{fig:Saturated}
\end{figure}

\section{Conclusion}
	
{In this work, we propose a family of novel continuous  P+d controllers to achieve finite-time convergence in bilateral teleoperation systems. These controllers are obtained from the energy-shaping methodology, and their design ensure that the closed-loop system admits an HA of negative degree. The resulting schemes can be designed for state-feedback or for output-feedback. Additionally, we also report controllers that can avoid actuator saturation. These schemes are simple to implement. Simulations and experiments validate the effectiveness of the controllers. Future work focuses on extending these controllers to handle delays in the communications and parameter uncertainty. }


\end{document}